\documentclass{article}
\usepackage[utf8]{inputenc}
\usepackage{url}
\usepackage{amsmath}
\usepackage{graphicx}
\usepackage{algorithm}
\usepackage{algpseudocode}
\usepackage{caption}
\usepackage{subcaption}
\usepackage[caption = false]{subfig}
\usepackage{color}
\usepackage{amssymb}
\usepackage{amsthm}
\usepackage{tensor}
\usepackage{paralist}
\usepackage{microtype}
\usepackage{authblk}

\title{Efficient differentially private learning improves drug sensitivity prediction}

\author{Antti Honkela$^{1,2,3,}$\footnote{These authors contributed equally to this work.}~$^{,}$\footnote{These authors jointly supervised the work.}~,
  Mrinal Das$^{4,*}$,
  Arttu Nieminen$^1$,
  Onur Dikmen$^1$ and
  Samuel Kaski$^{4,\dagger}$\\
  $^1$ Helsinki Institute for Information Technology HIIT, Department
  of Computer Science, University of Helsinki, Finland \\
  $^2$ Department of Mathematics and Statistics,
  University of Helsinki, Finland \\
  $^3$ Department of Public Health, University of Helsinki, Finland \\
  $^4$ Helsinki Institute for Information Technology HIIT, Department
  of Computer Science, Aalto University, Finland}

\date{}

\newcommand{\mechanism}{\mathcal{M}}
\newcommand{\dataset}{\mathcal{D}}
\newcommand{\R}{\mathbb{R}}

\newcommand{\mbb}{\mathbb}

\newtheorem{definition}{Definition}
\newtheorem{theorem}{Theorem}

\input{widebar.tex}

\begin{document}

\maketitle

\begin{abstract}
Users of a personalised recommendation system face a dilemma:
recommendations can be improved by learning from data, but only if the
other users are willing to share their private information. Good
personalised predictions are vitally important in precision medicine,
but genomic information on which the predictions are based is also
particularly sensitive, as it directly identifies the patients and
hence cannot easily be anonymised. Differential
privacy~\cite{Dwork2006,Dwork2014} has emerged as a potentially promising solution:
privacy is considered sufficient if presence of individual patients
cannot be distinguished. However, differentially private learning with
current methods does not improve
predictions with feasible data sizes and dimensionalities~\cite{Fredrikson2014}. Here we
show that useful predictors can be learned under powerful
differential privacy guarantees, and even from moderately-sized data
sets, by demonstrating significant improvements with a new robust
private regression method in the accuracy of private drug
sensitivity prediction~\cite{Costello2014}. The method combines two key
properties not present even in recent proposals~\cite{Wu2015,Foulds2016}, 
which can be generalised to other predictors: we prove it is asymptotically
consistently and efficiently private, and demonstrate that it performs
well on finite data. Good finite data performance is achieved by
limiting the sharing of private information by decreasing the
dimensionality and by projecting
outliers to fit tighter bounds, therefore needing to add
less noise for equal privacy.
As already the simple-to-implement method
shows promise on the challenging genomic data, we anticipate rapid
progress towards practical applications in many fields, such as mobile
sensing and social media, in addition to the badly needed precision
medicine solutions.
\end{abstract}

\section{Introduction}

The widespread collection of private data, both by individuals and
hospitals in the health domain, creates a major opportunity to develop
new services by learning predictive models from the data. Privacy-preserving
algorithms are required and have been proposed, but for instance 
anonymisation approaches~\cite{Bayardo2005,Machanavajjhala2007,Li2007}
cannot guarantee privacy against
adversaries with additional side information, and are poorly
suited for genomic data where the entire data is
identifying~\cite{Gymrek2013}. Guarantees of differential
privacy~\cite{Dwork2006,Dwork2014} remain valid even under these
conditions~\cite{Dwork2014}, and differential privacy has arisen as the
most popularly studied strong privacy mechanism for learning from data.

\section{Efficient differentially private learning}

Differential privacy~\cite{Dwork2006,Dwork2014} is a formulation of
reasonable privacy guarantees for privacy-preserving computation.
It gives guarantees about the output of a computation and can be
combined with complementary cryptographic approaches such as
homomorphic encryption~\cite{Gentry2009} if the computation process
needs protection too. An
algorithm $\mechanism$ operating on a data set $\dataset$ is said to
be \emph{differentially private} if for any two data sets $\dataset$
and $\dataset'$, differing only by one sample, the ratio of probabilities
of obtaining any specific result $c$ is bounded as
\begin{equation}
  \label{eq:dp_condition}
  \frac{p(\mechanism(\dataset) = c)}{p(\mechanism(\dataset') = c)}
  \le \exp(\epsilon).
\end{equation}
Because of symmetry between $\dataset$ and $\dataset'$ the
probabilities need to be similar to satisfy the condition.
Differential privacy is preserved in post-processing, which makes it
flexible to use in complex algorithms.
The $\epsilon$ is a privacy parameter interpretable as a privacy
budget, with higher values corresponding to less privacy preservation.
Differentially private learning algorithms are usually based on
perturbing either the input~\cite{Blum2005,Dwork2006},
output~\cite{Dwork2006,Wu2015} or the
objective~\cite{Chaudhuri2008,Zhang2012}.

Here we apply differential privacy to regression. The aim is to
learn a model to predict the scalar target $y_i$ from
$d$-dimensional inputs $x_i$ (Fig.~\ref{fig:illustration}a) as $y_i =
f(x_i) + \eta_i$, where $f$ is 
an unknown mapping and $\eta_i$ represents noise and modelling error.
We wish to design a suitable structure for $f$ and a differentially
private mechanism for efficiently learning an accurate private $f$ from a
data set $\dataset = \{(x_i, y_i)\}_{i=1}^n$.

We argue that a practical differentially private algorithm needs to
combine two things: (i) it needs to provide \emph{asymptotically
  efficiently private estimators} so that the excess loss incurred from
preserving privacy will diminish as the number of samples $n$ in the
data set increases; (ii) it needs to \emph{perform well on
  moderately-sized data}.

While the first requirement of asymptotic efficiency or consistency
seems obvious, it
is non-trivial to implement in practice and rules out some mechanisms
published even quite recently~\cite{Zhang2016}.  The requirement was
addressed in the Bayesian setting very recently~\cite{Foulds2016}, but
the method failed to cover the second equally important criterion.
Asymptotically consistently private methods always allow reaching stronger
privacy with more samples.

It is difficult to prove optimality of a method on finite data so good
performance needs to be demonstrated empirically. A design strategy
for good methods controls the amount of shared private information.
This has two components: (a) dimensionality needs to be reduced, to avoid
the inherent incompatibility of privacy and high dimensionality which
has been discussed previously~\cite{Duchi2014}, and (b) introducing
robustness by bounding and transforming each variable
(feature) to a tighter interval. Controlling
the amount of shared information also introduces a trade-off: compared
to the non-private setting, decreasing the dimensionality a lot may
degrade the performance of the non-private approach, while a
corresponding low-dimensional private algorithm may attain higher
performance than a higher-dimensional one (see the results and
Fig.~\ref{fig:tensor}a).

The essence of differential privacy is to inject a sufficient amount
of noise to mask the differences between the computation results
obtained from neighbouring data sets (differing by only one entry). The
definition depends on the worst-case behaviour, which implies that
suitably limiting the space of allowed results will reduce the amount
of noise needed and potentially improve the results.  In the output
perturbation framework this can be achieved by bounding the possible
outputs~\cite{Wu2015}.

Here we propose a more powerful approach of bounding the data by
projecting outliers to tighter bounds. The current
standard practice in private learning is to linearly transform the
data to desired
bounds~\cite{Zhang2012}. This is clearly sub-optimal as a few outliers
can force a very small scale for other points. Significantly higher
signal-to-privacy-noise ratio can be achieved by setting the bounds to
cover the essential variation in the data and projecting the outliers
separately inside these bounds. This approach also robustifies the
analysis against outliers as the projection can be made independent of
the outlier scale. In linear regression we call the resulting model
\emph{robust private linear regression}. It is illustrated in
Fig.~\ref{fig:illustration}b, c.

\begin{figure}
  \centering
  \raisebox{67mm}{\textbf{a}}
  \includegraphics{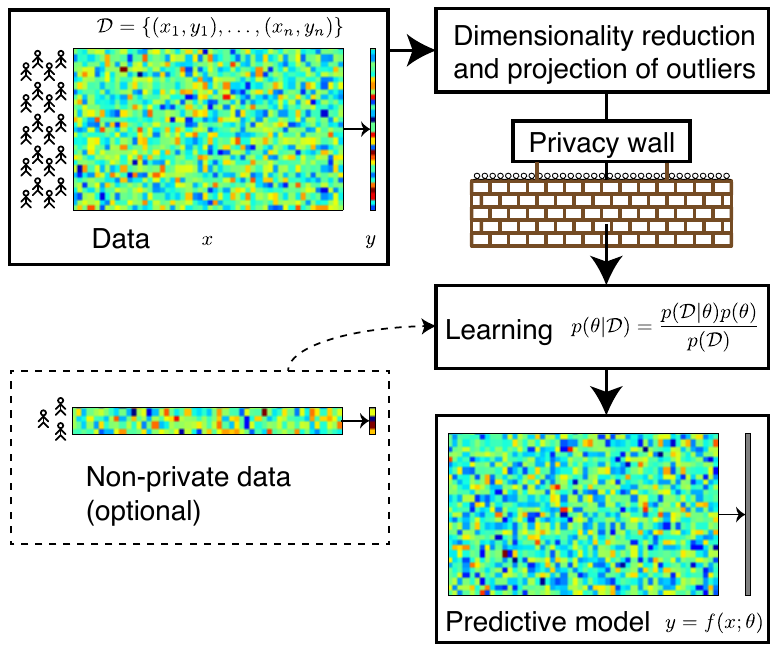}\\
  \raisebox{27mm}{\textbf{c}}\hspace{-3mm}
  \raisebox{57mm}{\textbf{b}}
  \includegraphics[trim=6mm 6mm 8mm 6mm,clip]{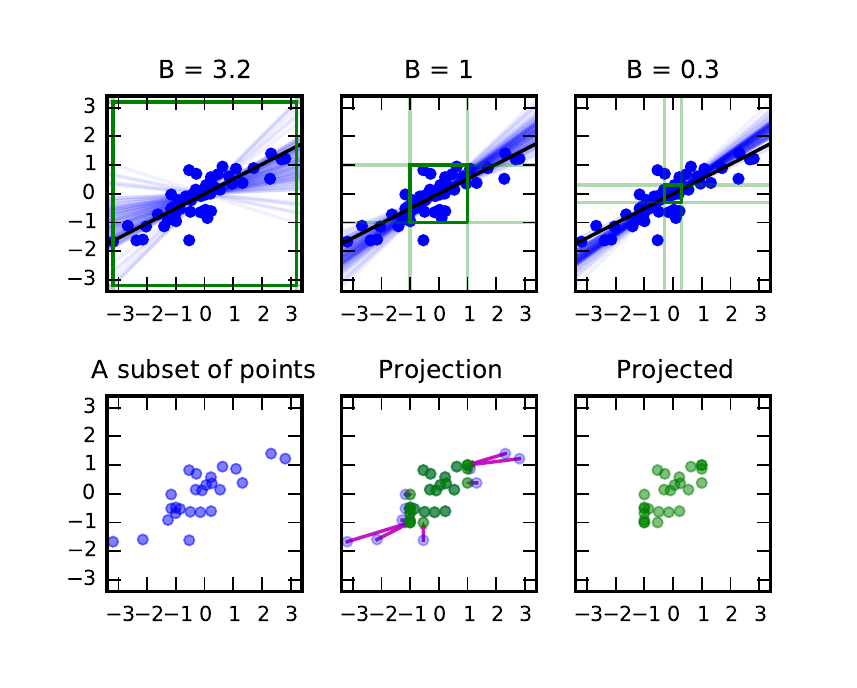}
  \caption{\textbf{Differentially private learning of a predictive
      model}. \textbf{a}, The modelling setup; most data (top) are
    available for learning only if their privacy can be
    protected. \textbf{b}, Bounding the data increasingly tightly (B;
    green square) brings 1D robust private linear regression models (blue lines
    illustrating the distribution of results of the randomised algorithm)
    closer to the non-private model (black line) as less noise needs to
    be injected. Blue points: data. \textbf{c}, The data are bounded
    in robust private linear regression
    by projecting outliers within the bounds (shown only for a subset
    of the points).}
  \label{fig:illustration}
\end{figure}

\section{Results}

Genomics is an important domain for privacy-aware modelling, in
particular for precision medicine. Many people wish to keep their and
also their relatives' genomes private~\cite{Naveed2015}, and simple
anonymisation is not sufficient to protect privacy since a genome is
inherently identifiable~\cite{Gymrek2013}. Furthermore, individual
genomes can be recovered from summary statistics~\cite{Homer2008} as
well as phenotype data such as gene expression
data~\cite{Harmanci2016}. On the other hand, previous research has
shown that poorly implemented private models may
put a patient to severe risk~\cite{Fredrikson2014}. 

We apply the robust private linear regression model to predict drug
sensitivity given gene expression data, in a setup where a small
internal data set can be complemented by a larger set only available
under privacy protection (Fig.~\ref{fig:illustration}a). We use data
from the Genomics of Drug Sensitivity in Cancer (GDSC)
project~\cite{Yang2013}, and the setting and evaluation are similar as
in the recent DREAM-NCI drug sensitivity prediction
challenge~\cite{Costello2014}. The sensitivity of each drug is
predicted with Bayesian linear regression based on expression of known
cancer genes identified by the GDSC project~\cite{Yang2013} to limit
the dimensionality. 
We achieve differential privacy
by injecting noise to the sufficient statistics computed from
the data, using the Laplace mechanism~\cite{Dwork2006}.
Full details are presented in Methods.

\begin{figure}
\centering
\centerline{\includegraphics{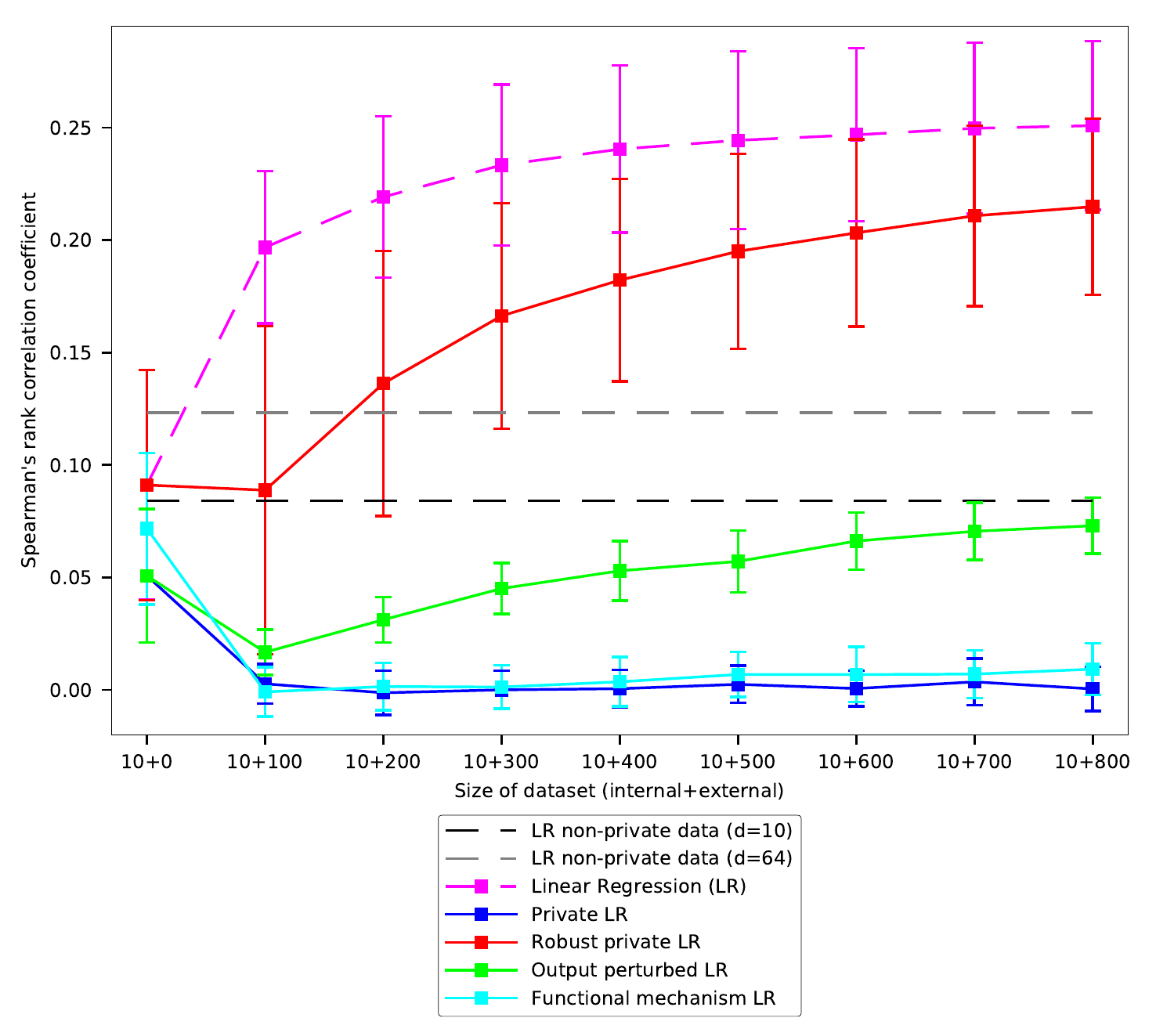}}
\caption{\textbf{Accuracy of drug sensitivity prediction in terms of
    Spearman's rank correlation coefficient over ranking
    cell lines by sensitivity to a drug (higher is
    better) increases with size of private data for the proposed
    robust private linear regression.} The state-of-the-art methods
  fail to utilise private data under strict privacy
  conditions. The baselines (horizontal dashed lines) are learned
  on 10 non-private
  data points; the private algorithms additionally have
  privacy-protected data (x-axis). The non-private algorithm (LR)
  has the same amount of additional non-privacy-protected data.
  All methods use 10-dimensional data except purple baseline
  showing the best performance with 10 non-private data points.
  Private methods use $\epsilon=2$, corresponding results for
  $\epsilon=1$ are in Fig.~\ref{fig:curves_eps1}. The results are averaged over all
  drugs and 50-fold Monte Carlo cross-validation; error bars denote
  standard deviation over 50 Monte Carlo repeats. (See Methods for details.)}
\label{fig:curves}
\end{figure}

Unlike with previous approaches, now prediction accuracy (ranking of
new cell lines~\cite{Costello2014} to sensitive vs insensitive
measured by Spearman's rank correlation;
Fig.~\ref{fig:curves}) improves when more privacy protected data
is received. The proposed non-linear projection of the data
to tighter bounds 
is the key to this success, as without it the
method performs as poorly as the earlier ones.

\begin{figure}
\centering 
\centerline{\includegraphics{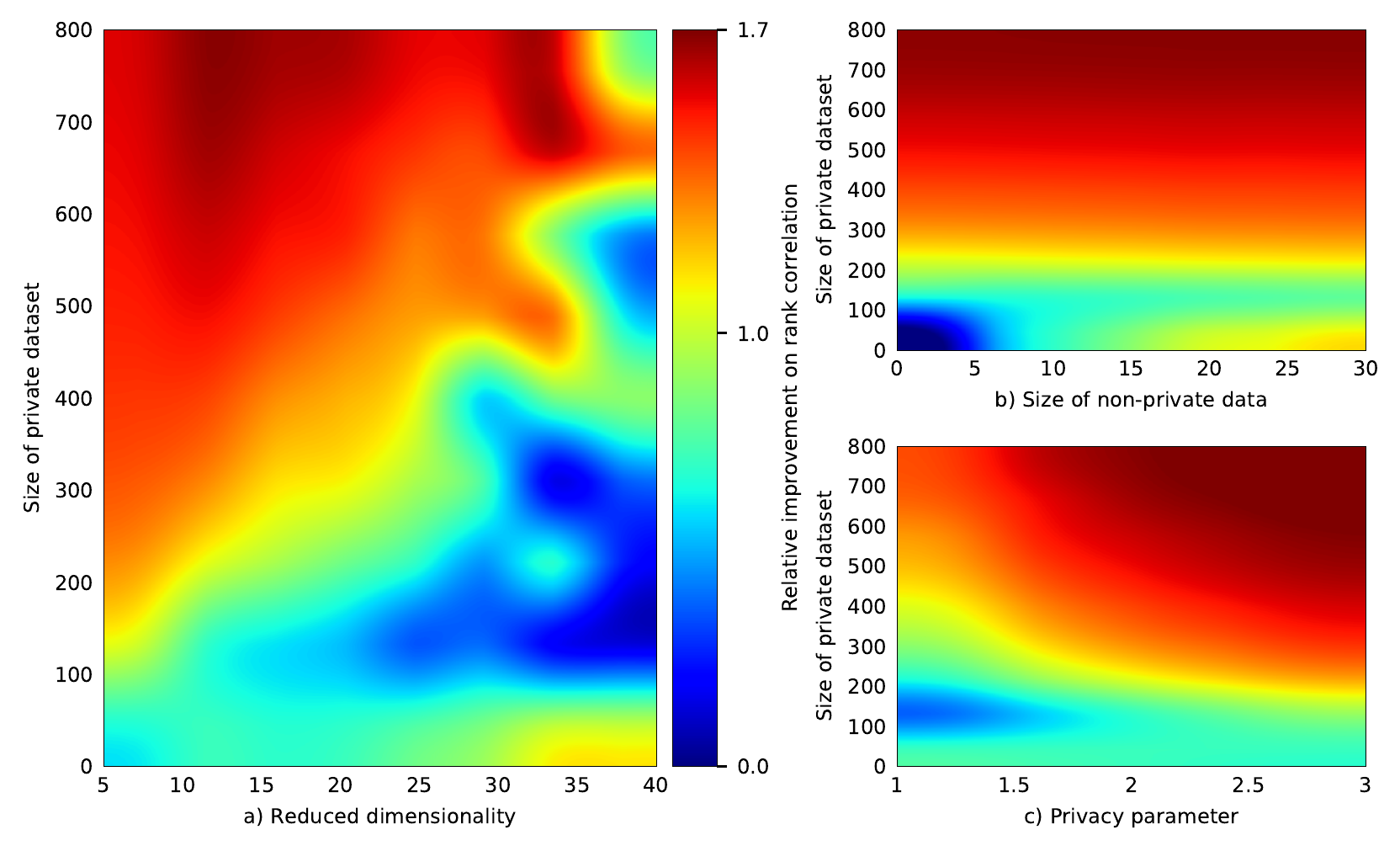}}
 \caption{\textbf{Key trade-offs in differentially private learning.}
 Relative improvements over baseline (10 non-private data points).
   \textbf{a}, As the dimensionality increases, the models without private
   data improve whereas more data are needed to
   improve performance of the private methods. \textbf{b}, With enough private data, adding
   more non-private data does not significantly increase the
   performance. \textbf{c}, More data are needed if privacy guarantees
   are tighter ($\epsilon$ is smaller). Size of non-private data is 10
   and $\epsilon=2$ (except when otherwise noted).
}
 \label{fig:tensor}
\end{figure}

To improve prediction performance in differentially private learning,
trade-offs need to be made between dimensionality and amount of data
(Fig.~\ref{fig:tensor}a), and between strength of privacy guarantees
and amount of data (Fig.~\ref{fig:tensor}c), but the amount of
optional non-private data matters significantly only when there is
very little private data (Fig.~\ref{fig:tensor}b).

In Secs.~\ref{sec:theor-backgr}--\ref{sec:diff-priv-line} in the 
Supplementary Information we define asymptotic consistency
and efficiency of private estimators relative to non-private ones
and prove that the optimal convergence rate of differentially private
Bayesian estimators to the corresponding non-private ones is
$\mathcal{O}(1/n)$ for $n$ samples, which is matched by our method.
Unlike existing approaches~\cite{Smith2008,Wasserman2010,Wang2015a}, we
compare the private estimators to the corresponding non-private ones, making the
theory more easily accessible and more broadly applicable.

Robust private linear regression treats non-private and scrambled
private data similarly in the model learning.
An interesting next step for further improving the accuracy on very
small private data would be to give a different weight to the clean
and privacy-scrambled data by incorporating knowledge of the injected
noise in the Bayesian inference, as has been proposed for generative
models~\cite{Williams2010}, but which is non-trivial in regression.

\section{Methods}

  \subsection{Linear regression model}

  The Bayesian linear regression model for scalar target
  $y_i$, with $d$-dimensional input $x_i$ and fixed noise precision
  $\lambda$, is defined by
  \begin{align}\label{eq:rr}
    y_i | x_i &\sim N(x_i^T \beta, \lambda) \nonumber \\
    \beta &\sim N(0, \lambda_0 I) ,
  \end{align}
  where $\beta$ is the unknown parameter to be learnt.
  The $\lambda$ and
  $\lambda_0$ are the precision parameters of the corresponding
  Gaussian distributions, and act as regularisers.

  Given an observed data set $\dataset = \{(x_i, y_i)\}_{i=1}^n$
  with sufficient statistics $n\widebar{xx} = \sum_{i=1}^n x_i x_i^T$
  and $n\widebar{xy} = \sum_{i=1}^n x_i y_i$, the posterior
  distribution of $\beta$ is Gaussian, $p(\beta | \dataset) =
  N(\beta;\; \mu_{*}, \Lambda_{*})$, with precision
\begin{equation}\label{eq:prec}  
  \Lambda_{*} = \lambda_0 I + \lambda n\widebar{xx}
\end{equation}
  and mean
\begin{equation}\label{eq:mean}  
  \mu_{*} = \Lambda_{*}^{-1} (\lambda n\widebar{xy})
\end{equation}

After learning with the training data set, the prediction of $y_i$
using $x_i$ is computed as follows:
\begin{equation}
\hat{y}_i = x_i^T \mu_*.
\end{equation}

A more robust alternative is to define prior distributions for the
precision parameters. In our case, a Gamma prior is assigned for both:
\begin{align}\label{eq:prior}
\lambda &\sim \mathrm{Gamma}(a,b)\nonumber\\
\lambda_0 &\sim \mathrm{Gamma}(a_0,b_0).
\end{align}

The posterior can be sampled using computational methods such as 
automatic differentation variational inference (ADVI) \cite{Kucukelbir2017}
where we fit a variational distribution to the posterior.
The precision parameters and correlation coefficients $\beta$ are then
sampled from the fitted distribution.
For this purpose, the data likelihood in
Eq.~\eqref{eq:rr} needs to be expressed in terms of the sufficient
statistics $n\widebar{xx}$, $n\widebar{xy}$, and
$n\widebar{yy}=\sum_{i=1}^n y_i^2$, which results in
\begin{multline}\label{eq:lh}
p(y|X,\beta,\lambda)=\left(\frac{\lambda}{2\pi}\right)^{n/2} \exp\left( -\frac{\lambda}{2}(\beta^T n\widebar{xx}\beta - 2\beta^T n\widebar{xy} + n\widebar{yy})\right).\\
\end{multline}
The prediction of $y_i$ is computed using $x_i$ and averaging over a
sufficiently large number $m$ of sampled regression coefficients
$\beta^{(k)}$ as
\begin{equation}\label{eq:predmcmc}
\hat{y}_i = \int p(y|\beta,X_{\mathrm{test},i})p(\beta|\dataset_{\mathrm{train}})d\beta \approx \frac{1}{m} \sum_{k=1}^m x_{\mathrm{test},i}^T \beta^{(k)}.
\end{equation}

For evaluation we keep a part of the data set $\dataset$ aside (not
used for training) and after predicting $\hat{y}_i$, we evaluate the
error between the actual $y_i$ and $\hat{y}_i$. In this paper, we
do this using Spearman's rank correlation coefficient to evaluate how
well the predictions separate sensitive and insensitive cell lines.

  \subsection{Differential privacy and efficiency}
  We apply differential privacy as defined in Eq.~\eqref{eq:dp_condition}.
  We use \emph{bounded differential privacy}, where two data sets are
  considered neighbouring if they contain the same number of elements
  $n$ with $n-1$ equal elements. Compared to the other common
  alternative of \emph{unbounded differential privacy}, in which two data
  sets are considered neighbouring if one is obtained from the other
  by adding or removing an element, bounded differential privacy makes
  it clear that the number of samples is not private which simplifies
  parameter tuning. The privacy parameter $\epsilon$ values are not
  directly comparable between the two formalisms, although an
  $\epsilon=k$ unbounded differentially private mechanism is always a
  $\epsilon=2k$ bounded differentially private mechanism.

  We define a private parameter estimation mechanism to be
  \emph{asymptotically consistently private}, if the private estimate converges
  in probability to the corresponding non-private estimate as the
  number of samples increases.  We show that the optimal rate of
  convergence of the private estimate to the corresponding non-private Bayesian
  estimate is $\mathcal{O}(1/n)$.  Mechanisms reaching this
  convergence rate are called \emph{asymptotically efficiently private}.
  A mechanism for estimating a model is called
  asymptotically consistently private with respect to a
  utility function if the utility of the private model converges in
  probability to the utility of the corresponding non-private model.
  For full detail of these definitions see Supplementary Information sections 1.1-1.2.

  \subsection{Robust private linear regression}
  The robust private linear regression is based on perturbing
  the sufficient statistics $n\widebar{xx} = \sum_{i=1}^n x_i x_i^T$,
  $n\widebar{xy} = \sum_{i=1}^n x_i y_i$, and $n\widebar{yy}=\sum_{i=1}^n y_i^2$.
  We use independent
  $p_i\epsilon$-differentially private Laplace mechanisms~\cite{Dwork2006} for
  perturbing each statistic with $\epsilon_i = p_i\epsilon$ for each $i=1,2,3$
  and $p_1+p_2+p_3=1$. Together, they provide an
  $\epsilon$-differentially private mechanism.

  We project the outliers in the private data sets to fit the data in the interval $[-B_*, B_*]$ as
  \begin{align} \label{eq:clip}
  x_{ij} &= \max(-B_x, \min(x_{ij},B_x)) \nonumber \\
  y_i &= \max(-B_y, \min(y_i,B_y)).
  \end{align}
  After the projection, $\| x_i \|_{\infty} \le
  B_x$ and $| y_i | \le B_y$, and we add noise to $n\widebar{xx}$ distributed
  as $\mathrm{Laplace}(0, b_{xx})$,
  to $n\widebar{xy} = \sum_{i=1}^n x_i y_i$ distributed as
  $\mathrm{Laplace}(0,b_{xy})$, and to $n\widebar{yy} = \sum_{i=1}^n y_i^2$
  distributed as $\mathrm{Laplace}(0,b_{yy})$,  
  where the scale parameters are
  $b_{xx}=\frac{d(d+1)B_x^2}{p_1\epsilon}$,
  $b_{xy}=\frac{2dB_xB_y}{p_2\epsilon}$, and
  $b_{yy}=\frac{B_y^2}{p_3\epsilon}$.
  This generalises earlier work on bounded variables~\cite{Foulds2016}
  to the unbounded case by introducing the projection.  Proof that
  this yields a valid asymptotically consistent and efficient
  differentially private mechanism is given in Supplementary Information
  section 2.
  We also show that a similar algorithm, applied to the estimation
  of a Gaussian mean, leads to an asymptotically consistent and
  efficient private
  estimate of the posterior mean, while the simpler input perturbation
  that perturbs the entire data set is not asymptotically consistently
  private.

  The privacy budget proportions $p_1,p_2,p_3$ and projection thresholds $B_x$, $B_y$ are important parameters for
  good model performance. As illustrated in Fig.~\ref{fig:clipping_effect},
  the projection thresholds depend strongly on the size of the data set. We
  propose finding the optimal parameter values on an auxiliary
  synthetic data set of the same size, which was found to be
  effective in our case.  We generate the auxiliary data set of $n$
  samples using a generative model similar to the one specified in
  Eq.~\eqref{eq:rr}:
  \begin{align}
   x_i & \sim N(0, I_d) \nonumber \\
   y_i | x_i &\sim N(x_i^T \beta, \lambda) \nonumber \\
   \beta & \sim N(0, \lambda_0 I),
  \end{align}
  where $d$ is the dimension. 
  
  First we find the optimal budget split $p_1,p_2,p_3$.
  For all possible combinations of $(p_1,p_2,p_3) \in \{0.05,0.1, \ldots ,
  0.90\}^3$, where $p_1 + p_2 + p_3 = 1$, we project the data using
  clipping thresholds 
  for the current split, and we
  perturb the sufficient statistics according to the current budget
  split. We compute the prediction 
  as in Eq.~\eqref{eq:predmcmc} using samples drawn from the variational distribution
  fitted with ADVI and compute the error with
  respect to the original values. The error measure we use is Spearman's
  rank correlation between the original and predicted values. The split $(p_1,p_2,p_3)$
  which gives the minimum error is used in all test settings. 
  As illustrated in Fig.~\ref{fig:budgetsplit}, in our experiments the optimal split 
  gives the largest proportion of the privacy budget to the term
  $n\widebar{xy}$ (60\%), the second largest proportion to the term
  $n\widebar{xx}$ (35\%), and the smallest possible proportion to the
  term $n\widebar{yy}$ (5\%).    
  
  We parameterise the projection thresholds
  as a function of the data standard deviation as
\begin{eqnarray}\label{eq:set_B}
&B_x= \omega_x \sigma_x,\quad  B_y= \omega_y \sigma_y  \\
&\omega_x,\omega_y \in \{0.1\omega\}_{\omega=1}^{20},
\end{eqnarray}
where the $\sigma_x$ and $\sigma_y$ are the standard deviations of $x$
(considering all dimensions) and $y$, respectively.
With all 400 pairs of $(B_x,B_y)$ as specified above, we apply the outlier
projection method of Eq.~\eqref{eq:clip}. We perturb the sufficient statistics 
according to the chosen optimal privacy budget split and fit the model as in 
Eq.~\eqref{eq:prec},\ref{eq:mean} using the
projected values and then compute the error with respect to the original
values. The pair of $(\omega_x,\omega_y)$ which
gives the minimum error is used to define the $(B_x,B_y)$ for the real data as in Eq.~\eqref{eq:set_B}. As the error we used Spearman's rank
correlation between original $y_{1:n}$ and predicted $\bar{y}_{1:n}$
based on the model learnt with projected values.

  \subsection{Data and pre-processing}

  We used the gene expression and drug sensitivity data from the
  \emph{Genomics of Drug Sensitivity in Cancer} (GDSC)
  project~\cite{Yang2013,Garnett2012} (release 6.1, March 2017,
  \url{http://www.cancerrxgene.org}) consisting of 265 drugs and a
  panel of 985 human cancer cell lines. The dimensionality of the RMA-normalised
  gene expression data was reduced from $d=17490$ down to 64
  based on prior knowledge about genes that are frequently mutated
  in cancer, provided by the GDSC project at \url{http://www.cancerrxgene.org/translation/Gene}.
  We further ordered the genes based on their mutation counts as
  reported at \url{http://cancer.sanger.ac.uk/cosmic/curation}.
  Drug responses were quantified by
  log-transformed IC50 values (the drug concentration yielding 50\%
  response) from the dose response data measured at 9 different
  concentrations.  The mean was first removed from each gene, 
  $x_{ij} := x_{ij}-\mathrm{mean}(x_{i:n,j})$, 
  and each data point was normalised to have L2-norm
  $\|x_i\|_2 = 1$, which focuses the analysis on relative expression
  of the selected genes, and equalises the contribution of each data
  point. The mean was removed from drug sensitivities, $y_i :=
  y_i-\mathrm{mean}(y_{1:n})$. Data with missing drug responses were
  ignored, making the number of cell lines different across different
  drugs.
  
  \subsection{Experimental setup}

  We carried out a 50-fold Monte Carlo cross-validation process for different splits of the data set into train and test using different
  random seeds. For each repeat, we randomly split the 985 cell lines
  to 100 for testing and the rest for training.
  We further randomly partitioned the training set to 30 non-private cell lines and used the
  rest as the private data set. In the experiments, we tested non-private
  data sizes from 0 to 30, and private data sizes from 100 to 800.
  The hyperparameters for the
  Gamma priors of precision parameters $\lambda,\lambda_0$ in Eq.~\eqref{eq:prior}
  were set to $a = b = a_0 = b_0 = 2$. The Gamma(2,2)
  distribution has mean 1 and variance 1/2 and defines
  a realistic distribution over sensible values of precision
  parameters which should be larger than zero. 
  We implemented the model and carried out the inference with 
  the PyMC3 Python module \cite{Salvatier2016}.
  Using ADVI, we fitted a normal distribution with uncorrelated variables 
  to the posterior distribution. We computed the drug response predictions
  using $m=5000$ samples from the fitted variational distribution.
  We used ADVI because it gives similar results as Hamiltonian Monte 
  Carlo sampling but significantly faster.
  The optimal privacy budget split was based on prediction
  performance averaged over five auxiliary data sets of 500 synthetic samples 
  (approximately half of the GDSC data set size) and five generated
  noise samples, and for each split, the optimal projection thresholds
  were chosen similarly based on average performance over five auxiliary
  data sets and five noise samples.  The prediction for each split was
  computed using $m=5000$ samples drawn from the variational distribution 
  fitted with ADVI. The final optimal
  projection thresholds for each test case were chosen using the
  optimal budget split and based on average prediction performance over
  20 auxiliary data sets and 20 noise samples.
  All auxiliary data sets were generated by fixing the precision 
  parameter values to the prior means, $\lambda=\lambda_0=1$. 
  The prediction for each pair of projection thresholds was also 
  computed using fixed precision parameters as in Eq.~\eqref{eq:prec} 
  and Eq.~\eqref{eq:mean}, as generating samples from the fitted 
  variational distribution for all test cases would have been
  infeasible in practice.

  \subsection{Alternative methods used in comparisons}

  We compared five models: (i) linear regression (LR) as defined in
  Eq.~\eqref{eq:rr}, (ii) robust private LR is the proposed method, and
  (iii) private LR is the proposed method without projection of the
  outliers, (iv) output perturbed LR~\cite{Wu2015}, and (v) functional
  mechanism LR~\cite{Zhang2012}.
  Output perturbed LR learns parameters
  $\beta$ using the same LR model in Eq.~\eqref{eq:rr}, but instead of
  statistics the parameters are perturbed, in a data-independent
  manner. Our implementation of output perturbed LR makes use of
  minConf optimisation package~\cite{Schmidt2009}.
  For functional mechanism LR we used the code publicly available at
  \url{https://sourceforge.net/projects/functionalmecha/}.

  \subsection{Alternative interpretation: transformed linear regression}

  The outlier projection mechanism can also be interpreted to produce a transformed
  linear regression problem,
  \begin{equation}
    \label{eq:transformationeq}
    \phi_y(y_i) | x_i \sim N(\phi_x(x_i)^T \beta, \lambda),
  \end{equation}
  where the functions $\phi_y()$ and $\phi_x()$ implementing the
  outlier projection can be defined as
  \begin{align}
    \phi_y(y_i) &= \max(-B_y, \min(B_y, y_i)) \\
    \phi_x(x_i) &= \max(-B_x, \min(B_x, x_i)).
  \end{align}
  The normalisation of data can also be included as a transformation.
  This interpretation makes explicit the flexibility in designing the
  transformations: the differential privacy guarantees will remain
  valid as long as the transformations obey the bounds
  \begin{equation}
    \label{eq:transformbounds}
    \phi_y(y_i) \in [-B_y, B_y], \quad \phi_x(x_i) \in [-B_x, B_x].
  \end{equation}

\subsubsection*{Acknowledgements}
We would like to thank Muhammad Ammad-ud-din for assistance in
data processing and Otte Heinävaara for assistance in the
theoretical analysis. 
We acknowledge the computational resources provided 
by the Aalto Science-IT project.
This work was funded by the Academy of Finland (Centre of
Excellence COIN; and grants 283193 (S.K. and M.D), 294238 and
292334 (S.K.), 278300 (A.H. and O.D.), 259440 and 283107 (A.H.)).

\bibliographystyle{abbrv}
\bibliography{dpdrugsens_refs}

\begin{thebibliography}{10}

\bibitem{Bayardo2005}
R.~Bayardo and R.~Agrawal.
\newblock Data privacy through optimal k-anonymization.
\newblock In {\em Proc. 21st Int. Conf. Data Eng. (ICDE 2005)}, 2005.

\bibitem{Blum2005}
A.~Blum, C.~Dwork, F.~McSherry, and K.~Nissim.
\newblock Practical privacy: the {SuLQ} framework.
\newblock In {\em Proc. PODS 2005}, 2005.

\bibitem{Chaudhuri2008}
K.~Chaudhuri and C.~Monteleoni.
\newblock Privacy-preserving logistic regression.
\newblock In {\em Adv. Neural Inf. Process. Syst. 21}, 2008.

\bibitem{Costello2014}
J.~C. Costello, L.~M. Heiser, E.~Georgii, M.~G{\"{o}}nen, M.~P. Menden, N.~J.
  Wang, M.~Bansal, M.~Ammad-ud din, P.~Hintsanen, S.~A. Khan, J.-P. Mpindi,
  O.~Kallioniemi, A.~Honkela, T.~Aittokallio, K.~Wennerberg, {NCI DREAM
  Community}, J.~J. Collins, D.~Gallahan, D.~Singer, J.~Saez-Rodriguez,
  S.~Kaski, J.~W. Gray, and G.~Stolovitzky.
\newblock A community effort to assess and improve drug sensitivity prediction
  algorithms.
\newblock {\em Nat. Biotechnol.}, 32(12):1202--1212, Dec 2014.

\bibitem{Diaconis1979}
P.~Diaconis and D.~Ylvisaker.
\newblock Conjugate priors for exponential families.
\newblock {\em Ann. Stat.}, 7(2):269–281, Mar 1979.

\bibitem{Duchi2014}
J.~C. Duchi, M.~I. Jordan, and M.~J. Wainwright.
\newblock Privacy aware learning.
\newblock {\em J. ACM}, 61(6):1--57, Dec 2014.

\bibitem{Dwork2006}
C.~Dwork, F.~McSherry, K.~Nissim, and A.~Smith.
\newblock Calibrating noise to sensitivity in private data analysis.
\newblock In {\em Proc. TCC 2006}. 2006.

\bibitem{Dwork2014}
C.~Dwork and A.~Roth.
\newblock The algorithmic foundations of differential privacy.
\newblock {\em Found. Trends Theor. Comput. Sci.}, 9(3-4):211--407, Aug. 2014.

\bibitem{Foulds2016}
J.~Foulds, J.~Geumlek, M.~Welling, and K.~Chaudhuri.
\newblock On the theory and practice of privacy-preserving {B}ayesian data
  analysis.
\newblock In {\em Proc.\ UAI 2016}, Mar. 2016.
\newblock arXiv:1603.07294.

\bibitem{Fredrikson2014}
M.~Fredrikson, E.~Lantz, S.~Jha, S.~Lin, D.~Page, and T.~Ristenpart.
\newblock Privacy in pharmacogenetics: An end-to-end case study of personalized
  warfarin dosing.
\newblock In {\em Proc. 23rd USENIX Security Symp. (USENIX Security 2014)},
  pages 17--32, 2014.

\bibitem{Garnett2012}
M.~J. Garnett et~al.
\newblock Systematic identification of genomic markers of drug sensitivity in
  cancer cells.
\newblock {\em Nature}, 483(7391):570--575, Mar 2012.

\bibitem{Gentry2009}
C.~Gentry.
\newblock {\em A fully homomorphic encryption scheme}.
\newblock PhD thesis, Stanford University, 2009.

\bibitem{Gymrek2013}
M.~Gymrek, A.~L. McGuire, D.~Golan, E.~Halperin, and Y.~Erlich.
\newblock Identifying personal genomes by surname inference.
\newblock {\em Science}, 339(6117):321--324, Jan 2013.

\bibitem{Harmanci2016}
A.~Harmanci and M.~Gerstein.
\newblock Quantification of private information leakage from phenotype-genotype
  data: linking attacks.
\newblock {\em Nat. Methods}, 13(3):251--256, Mar 2016.

\bibitem{Homer2008}
N.~Homer et~al.
\newblock Resolving individuals contributing trace amounts of {DNA} to highly
  complex mixtures using high-density {SNP} genotyping microarrays.
\newblock {\em PLoS Genet.}, 4(8):e1000167, Aug 2008.

\bibitem{Kucukelbir2017}
A.~Kucukelbir, D.~Tran, R.~Ranganath, A.~Gelman, and D.~M. Blei.
\newblock Automatic differentiation variational inference.
\newblock {\em J Mach Learn Res}, 18(14):1--45, 2017.

\bibitem{Li2007}
N.~Li, T.~Li, and S.~Venkatasubramanian.
\newblock t-closeness: Privacy beyond k-anonymity and l-diversity.
\newblock In {\em Proc. ICDE 2007}, 2007.

\bibitem{Machanavajjhala2007}
A.~Machanavajjhala, D.~Kifer, J.~Gehrke, and M.~Venkitasubramaniam.
\newblock L-diversity: Privacy beyond k-anonymity.
\newblock {\em TKDD}, 1(1):3, Mar 2007.

\bibitem{Naveed2015}
M.~Naveed et~al.
\newblock Privacy in the genomic era.
\newblock {\em ACM Comput. Surv.}, 48(1):1--44, Aug 2015.

\bibitem{Salvatier2016}
J.~Salvatier, T.~V. Wiecki, and C.~Fonnesbeck.
\newblock Probabilistic programming in {P}ython using {PyMC}3.
\newblock {\em {PeerJ} Computer Science}, 2:e55, apr 2016.

\bibitem{Schmidt2009}
M.~Schmidt, E.~van~den Berg, M.~Friedlander, and K.~Murphy.
\newblock Optimizing costly functions with simple constraints: A limited-memory
  projected quasi-newton algorithm.
\newblock In {\em Proc. AISTATS 2009}, 2009.

\bibitem{Smith2008}
A.~Smith.
\newblock Efficient, differentially private point estimators.
\newblock Sept. 2008.
\newblock arXiv:0809.4794 [cs.CR].

\bibitem{Wang2015a}
Y.-X. Wang, J.~Lei, and S.~E. Fienberg.
\newblock Learning with differential privacy: Stability, learnability and the
  sufficiency and necessity of {ERM} principle.
\newblock Feb. 2015.
\newblock arXiv: 1502.06309 [stat.ML].

\bibitem{Wasserman2010}
L.~Wasserman and S.~Zhou.
\newblock A statistical framework for differential privacy.
\newblock {\em J. Am. Stat. Assoc.}, 105(489):375--389, Mar 2010.

\bibitem{Williams2010}
O.~Williams and F.~McSherry.
\newblock Probabilistic inference and differential privacy.
\newblock In {\em Adv. Neural Inf. Process. Syst. 23}, 2010.

\bibitem{Wu2015}
X.~Wu, M.~Fredrikson, W.~Wu, S.~Jha, and J.~F. Naughton.
\newblock Revisiting differentially private regression: Lessons from learning
  theory and their consequences.
\newblock Dec. 2015.
\newblock arXiv:1512.06388 [cs.CR].

\bibitem{Yang2013}
W.~Yang et~al.
\newblock Genomics of drug sensitivity in cancer ({GDSC}): a resource for
  therapeutic biomarker discovery in cancer cells.
\newblock {\em Nucleic Acids Res.}, 41(Database issue):D955--D961, Jan 2013.

\bibitem{Zhang2012}
J.~Zhang, Z.~Zhang, X.~Xiao, Y.~Yang, and M.~Winslett.
\newblock Functional mechanism: Regression analysis under differential privacy.
\newblock {\em {PVLDB}}, 5(11):1364--1375, 2012.

\bibitem{Zhang2016}
Z.~Zhang, B.~Rubinstein, and C.~Dimitrakakis.
\newblock On the differential privacy of {B}ayesian inference.
\newblock In {\em Proc. AAAI 2016}, 2016.

\end{thebibliography}

\clearpage

\appendix

\section*{Supplementary Information}

\section{Theoretical background}
\label{sec:theor-backgr}

We argue that 
effective differentially private predictive modelling methods can be
developed by a combination of:
\begin{enumerate}[i.]
\item An asymptotically efficiently private mechanism for which the
  effect of the noise added to guarantee privacy vanishes as the
  number of samples increases; and \label{hypot1}
\item A way to limit the amount of private information to be shared.
  This yields better performance on finite data as less noise needs to
  be added for equivalent privacy. This can be achieved through
  a combination of two things: \label{hypot2}
  \begin{enumerate}[a.]
  \item An approach to decrease the dimensionality of the data prior to
    the application of the private algorithm; and \label{hypot2a}
  \item A method to focus the privacy guarantees to relevant variation
    in data. \label{hypot2b}
  \end{enumerate}
\end{enumerate}

Criterion \ref{hypot1} can be formally stated through additional loss
in accuracy or utility of the estimates because of privacy.  Our main
asymptotic result is that the optimal convergence rate of a
differentially private mechanism to a Bayesian estimate is
$\mathcal{O}(1/n)$, which can be reached by our proposed mechanism.

Criterion \ref{hypot2} is non-asymptotic and thus more difficult to
address theoretically.  It manifests itself in the constants in the
convergence rates as well as empirical findings on the effect of
dimensionality reduction and projecting outliers to tighter bounds as
discussed in the main text and in Fig.~\ref{fig:clipping_effect}.

\subsection{Definition of asymptotic efficiency}

We begin by formalisation of the theory behind Criterion \ref{hypot1}.

\begin{definition}
  A differentially private mechanism $\mechanism$ is
  \emph{asymptotically consistent with respect to an estimated
  parameter} $\theta$ if the private estimates $\hat{\theta}_{\mechanism}$
  given a data set $\dataset$ converge in probability to the
  corresponding non-private estimates $\hat{\theta}_{NP}$ as the
  number of samples, $n = |\dataset|$, grows without bound, i.e., if for
  any\footnote{We use $\alpha$ in limit expressions instead of usual $\epsilon$ to avoid confusion with $\epsilon$-differential privacy.} $\alpha > 0$,
  $$ \lim\limits_{n \rightarrow \infty}
  \mathrm{Pr}\{\|\hat{\theta}_{\mechanism} - \hat{\theta}_{NP}\| > \alpha\} = 0. $$
\end{definition}

\begin{definition}
  A differentially private mechanism $\mechanism$ is
  \emph{asymptotically efficiently private with respect to an estimated
    parameter} $\theta$, if the mechanism is asymptotically consistent and the
  private estimates $\hat{\theta}_{\mechanism}$ converge to
  the corresponding non-private estimates $\hat{\theta}_{NP}$
  at the rate $\mathcal{O}(1/n)$, i.e., if for any $\alpha > 0$ there exist
  constants $C, N$ such that
  $$ 
  \mathrm{Pr}\{\|\hat{\theta}_{\mechanism} - \hat{\theta}_{NP}\| > C / n \} < \alpha $$
  for all $n \ge N$.
\end{definition}

The term asymptotically efficiently private in the above definition is
justified by the following theorem, which shows that the rate
$\mathcal{O}(1/n)$ is
optimal for estimating expectation parameters of exponential family
distributions.  As it seems unlikely that better
rates could be obtained for more difficult problems, we conjecture
that this rate cannot be beaten for Bayesian estimates in general.

\begin{theorem}
  The private estimates $\hat{\theta}_{\mechanism}$ of an exponential
  family posterior expectation parameter $\theta$, generated by a
  differentially private mechanism $\mechanism$ that achieves
  $\epsilon$-differential privacy for any $\epsilon > 0$, cannot
  converge to the corresponding non-private estimates
  $\hat{\theta}_{NP}$ at a rate faster than $1/n$.  This is,
  assuming $\mechanism$ is $\epsilon$-differentially private,
  there exists no function $f(n)$ such that
  $\lim\sup n f(n) = 0$ and for all $\alpha > 0$, there exists
  a constant $N$ such that
  $$ \mathrm{Pr}\{\|\hat{\theta}_{\mechanism} - \hat{\theta}_{NP}\| > f(n)
  \} < \alpha $$ for all $n \ge N$.
\end{theorem}

\begin{proof}
  The non-private estimate of an expectation parameter of an
  exponential family is~\cite{Diaconis1979}
  \begin{equation}
    \label{eq:exp_parameter}
    \hat{\theta}_{NP} | x_1, \dots, x_n = \frac{n_0 x_0 + \sum_{i=1}^n x_i}{n_0 + n}.
  \end{equation}
  The difference of the estimates from two neighbouring data sets
  differing by one element is
  \begin{equation}
    \label{eq:est_diff}
    (\hat{\theta}_{NP} | \dataset) - (\hat{\theta}_{NP} | \dataset')
    = \frac{x - y}{n_0 + n},
  \end{equation}
  where $x$ and $y$ are the corresponding mismatched elements.
  Let $\Delta = \max(\| x - y \|)$, and let $\dataset$ and $\dataset'$
  be neighbouring data sets including these maximally different
  elements.

  Let us assume that there exists a function $f(n)$ such that
  $\lim\sup n f(n) = 0$ and for all $\alpha > 0$ there exists
  a constant $N$ such that
  $$ \mathrm{Pr}\{\|\hat{\theta}_{\mechanism} - \hat{\theta}_{NP}\| > f(n)
  \} < \alpha $$ for all $n \ge N$.

  Fix $\alpha > 0$ and choose $M \ge \max(N, n_0)$ such that $f(n)
  \le \Delta / 4n$ for all $n \ge M$. This implies that
  \begin{equation}
    \label{eq:estimatediff}
    \| (\hat{\theta}_{NP} | \dataset) - (\hat{\theta}_{NP} | \dataset') \|
    = \frac{\Delta}{n_0 + n} \ge \frac{\Delta}{2n} \ge 2 f(n).
  \end{equation}

  Let us define the region $C_\dataset = \{ t \;|\; \| (\hat{\theta}_{NP} | \dataset) - t \| < f(n) \}. $
  Based on our assumptions we have
  \begin{align}
    \mathrm{Pr}(\hat{\theta}_{\mechanism} | \dataset \in C_\dataset) &> 1 - \alpha \\
    \mathrm{Pr}(\hat{\theta}_{\mechanism} | \dataset' \in C_\dataset) &< \alpha
  \end{align}
  which implies that
  \begin{equation}
    \label{eq:dp_test}
    \frac{\mathrm{Pr}(\hat{\theta}_{\mechanism} | \dataset \in C_\dataset)}
    {\mathrm{Pr}(\hat{\theta}_{\mechanism} | \dataset' \in C_\dataset)}
    > \frac{1 - \alpha}{\alpha}
  \end{equation}
  which means that $\mechanism$ cannot be differentially private with
  $\epsilon < \log \left((1 - \alpha)/\alpha\right) \rightarrow \infty$ as
  $\alpha \rightarrow 0$.
\end{proof}

\subsection{Different utility functions}

\begin{definition}
Let $\mathcal{U}(\hat{\theta}_{NP}(\dataset))$ measure
the utility of the non-private model $\hat{\theta}_{NP}$ estimated from
data set $\dataset$ and let
$\mathcal{U}(\hat{\theta}_{\mechanism}(\dataset))$ measure the
corresponding utility of the private model $\hat{\theta}_{\mechanism}$
obtained using differentially private mechanism $\mechanism$.  The
mechanism $\mechanism$ is \emph{asymptotically consistent with respect
to a bounded utility} $\mathcal{U}$ if the random variables
$\mathcal{U}(\hat{\theta}_{\mechanism}(\dataset))$ converge in
probability to $\mathcal{U}(\hat{\theta}_{NP}(\dataset))$ as the number of
samples, $n = |\dataset|$, grows without bound, i.e., if for any
$\alpha > 0$,
$$ \lim\limits_{n \rightarrow \infty} \mathrm{Pr}\{|\mathcal{U}(\hat{\theta}_{\mechanism}(\dataset)) - \mathcal{U}(\hat{\theta}_{NP}(\dataset))| > \alpha\} = 0. $$
\end{definition}

\begin{theorem}
  A differentially private mechanism $\mechanism$ that is
  asymptotically consistent with respect to a set of parameters is
  asymptotically consistent with respect to any continuous utility
  that only depends on those parameters.
\end{theorem}

\begin{proof}
  If $\hat{\theta}_{\mechanism}$ converges in probability to
  $\hat{\theta}_{NP}$ then by the continuous mapping theorem the
  value of $\mathcal{U}(\hat{\theta}_{\mechanism})$ will converge in
  probability to $\mathcal{U}(\hat{\theta}_{NP})$.
\end{proof}

\subsection{Example: Gaussian mean}

\begin{theorem}\label{thm:efficiency_of_gaussian_mean}
  Differentially private inference of the mean of a Gaussian variable,
  with Laplace mechanism to perturb the sufficient statistics, is
  asymptotically consistent with respect to the posterior mean.
\end{theorem}

\begin{proof}
  Let us consider the model
  \begin{align*}
    x_i &\sim N(\mu, \Lambda)\\
    \mu &\sim N(\mu_0, \Lambda_0)
  \end{align*}
  with $\mu$ as the unknown parameter and $\Lambda$ and $\Lambda_0$
  denoting the fixed prior precision matrices of the noise and the mean,
  respectively.
  We assume $||x_i||_1 \le B$ and enforce this by projecting
  the larger elements to satisfy this bound.

  Let the observed data set be $\dataset = \{x_i\}_{i=1}^n$ with
  sufficient statistic $n \bar{x} = \sum_{i=1}^n x_i$.

  The non-private posterior mean is
  $$ \mu_{NP} = (\Lambda_0 + n \Lambda)^{-1} (\Lambda n \bar{x} + \Lambda_0 \mu_0). $$
  The corresponding private posterior mean is obtained by replacing
  $n\bar{x}$ with the perturbed version
  $n\bar{x}' = n\bar{x} + \delta$, where $\delta=(\delta_1,\dots,\delta_d)^T \in \R^d$ with 
  $\delta_j \sim \mathrm{Laplace}(0, \frac{2Bd}{\epsilon})$ and $d =
\dim(x_i)$, yielding
  $$ \mu_{DP} = (\Lambda_0 + n \Lambda)^{-1} (\Lambda (n \bar{x} + \delta) + \Lambda_0 \mu_0). $$

  The difference of the private and non-private means is
  \begin{align*}
    \|\mu_{DP} - \mu_{NP}\|_1 &= \| (\Lambda_0 + n \Lambda)^{-1} (\Lambda \delta) \|_1 \\
    &= \| (\Lambda^{-1} \Lambda_0 + n \cdot I)^{-1} \delta \|_1
      \le \frac{c}{n} \| \delta \|_1,
  \end{align*}
  which is valid for all $c > 1$ for large enough $n$.
  This implies that
  $$ \mathrm{Pr}\{\|\mu_{DP} - \mu_{NP}\|_1 \ge \alpha\} \le
  \mathrm{Pr}\left\{\frac{c}{n} \|\delta\|_1 \ge \alpha\right\} \rightarrow 0 $$ as
  $n \rightarrow \infty$ for all $\alpha > 0$.
\end{proof}

\begin{theorem}
  Differentially private inference of the mean of a Gaussian variable
  with Laplace mechanism to perturb the input data set (naive input
  perturbation) is \emph{not} asymptotically consistent with respect
  to the posterior mean.
\end{theorem}

\begin{proof}
  The mechanism is almost the same as in
  Theorem~\ref{thm:efficiency_of_gaussian_mean},
  but we now have $n\bar{x}' = n\bar{x} + \sum_{i=1}^n \delta_i)$ where $\delta_i=(\delta_{i1},\dots,\delta_{id})^T \in \R^d$ with
  $\delta_{ij} \sim \mathrm{Laplace}(0, \frac{2Bd}{\epsilon})$.  Similar computation as above
  yields
  \begin{align*}
    \|\mu_{DP} - \mu_{NP}\|_1 &= \left\| (\Lambda_0 + n \Lambda)^{-1} (\Lambda \sum_{i=1}^n \delta_i) \right\|_1 \\
    &= \left\| (\frac{1}{n}\Lambda^{-1} \Lambda_0 + I)^{-1} \frac{1}{n} \sum_{i=1}^n \delta_i \right\|_1
      \ge \frac{1}{2} \left\| \frac{1}{n} \sum_{i=1}^n \delta_i \right\|_1
  \end{align*}
  for sufficiently large $n$.  By the central limit theorem the
  distribution of $\frac{1}{n} \sum_{i=1}^n \delta_i$ converges to a
  Gaussian with non-zero variance.
  Hence $\mu_{DP}$ does not converge to $\mu_{NP}$ for
  large $n$ and the method is not asymptotically consistent.
\end{proof}

\subsubsection{Asymptotic efficiency}

\begin{theorem}\label{thm:rate_gaussian_mean}
  $\epsilon$-differentially private estimate of the mean of a
  $d$-dimensional Gaussian variable $x$ bounded by $\|x_i\|_1 \le B$
  in which the Laplace mechanism is used to perturb the sufficient statistics,
  is asymptotically efficiently private.
\end{theorem}

\begin{proof}
In the proof of Theorem~\ref{thm:efficiency_of_gaussian_mean} we
showed that
$$ \|\mu_{DP} - \mu_{NP}\|_1 \le \frac{c}{n} \| \delta \|_1, $$
where $\delta = (\delta_1, \dots, \delta_d)^T \in \R^D$ with $\delta_j \sim \mathrm{Laplace}\left(0,
\frac{2Bd}{\epsilon}\right)$.

Because $\delta_j$ is Laplace, $|\delta_j|$ is exponential with
$$ |\delta_j| \sim \mathrm{Exponential}\left(\frac{\epsilon}{2Bd}\right) $$
and
$$ \|\delta\|_1 = \sum_{j=1}^d |\delta_j| \sim \mathrm{Gamma}\left(d, \frac{\epsilon}{2Bd}\right). $$

Given $\alpha > 0$ we can choose $C > c F^{-1}(1-\alpha; d, \epsilon/2Bd)$,
where $F^{-1}(x; a, b)$ is the inverse cumulative distribution function of
the Gamma distribution with shape $a$ and rate $b$,
to ensure that 
\begin{equation}
  \label{eq:gaussmean_prob}
  \mathrm{Pr}\left\{\|\mu_{DP} - \mu_{NP}\|_1 > \frac{C}{n} \right\}
  \le \mathrm{Pr}\left\{ \frac{1}{n} \|\delta\|_1 > \frac{C}{n} \right\}
  = \mathrm{Pr}\{ \|\delta\|_1 > C \}
  < \alpha.
\end{equation}
\end{proof}

\subsubsection{Convergence rate}

We can further study the probability of making an error of at least
a given magnitude as
\begin{align}
  \mathrm{Pr}\{\|\mu_{DP} - \mu_{NP}\|_1 \ge \phi\}
  & \le \mathrm{Pr}\left\{\frac{c}{n} \|\delta\|_1 \ge \phi\right\} \nonumber\\
  & = \mathrm{Pr}\left\{\mathrm{Gamma}\left(d, \frac{n \epsilon}{2Bcd}\right) \ge \phi\right\} \nonumber\\
  & = 1 - F\left(\phi; d, \frac{n \epsilon}{2Bcd}\right)
    = 1 - \frac{\gamma(d, \frac{n \phi \epsilon}{2Bcd})}{\Gamma(d)},
  \label{eq:gauss_rate}
\end{align}
where $F(x; a, b)$ is the cumulative distribution function of
the Gamma distribution with shape $a$ and rate $b$.

The formula in Eq.~\eqref{eq:gauss_rate} unfortunately has no simple
closed form expression.
The result shows, however, that the $n$ required to reach a certain
level of performance is linear in $B$ and $\frac{1}{\epsilon}$.  The
dependence on $d$ is complicated, but it is in general super-linear as
suggested by the mean of the gamma distribution in
Eq.~\eqref{eq:gauss_rate}, $\frac{2Bd^2}{n \epsilon}$.

\subsection{Example: Zhang et al., AAAI 2016, (arxiv:1512.06992)}

In their paper Zhang et al.\ derive utility bounds for a number of
mechanisms.  The bounds are clearly insufficient to demonstrate
the asymptotic efficiency of the corresponding methods. For Laplace
mechanism applied to Bayesian network inference, their bound on
excess KL-divergence as a function of the data set size $n$ is
\begin{equation*}
  \mathcal{O}(m n \ln n) \left[1 - \exp\left(-\frac{n\epsilon}{2 |\mathcal{I}|}\right)  \right] + \sqrt{- \mathcal{O}(m n \ln n) \ln \delta}.
\end{equation*}

\section{Differentially private linear regression}
\label{sec:diff-priv-line}

Let us next consider the linear regression model with fixed noise
$\Lambda$,
\begin{align*}
  y_i | x_i &\sim N(x_i^T \beta, \Lambda)\\
  \beta &\sim N(\beta_0, \Lambda_0),
\end{align*}
with $\beta$ as the unknown parameter and $\Lambda$ and $\Lambda_0$
denoting the precision matrices of the corresponding distributions.

Let the observed data set be $\dataset = \{(x_i, y_i)\}_{i=1}^n$ with
sufficient statistics $n\widebar{xx} = \sum_{i=1}^n x_i x_i^T$ and
$n\widebar{xy} = \sum_{i=1}^n x_i y_i$.

The non-private posterior precision of $\beta$ is
$$ \Lambda_{NP} = \Lambda_0 + \Lambda n\widebar{xx} $$
and the corresponding posterior mean is
\begin{equation}
  \label{eq:linreg_np}
  \mu_{NP} = \Lambda_{NP}^{-1} (\Lambda n\widebar{xy} + \Lambda_0 \beta_0).
\end{equation}
The corresponding private posterior precision is obtained by replacing
$n\widebar{xx}$ with the perturbed version $n\widebar{xx}' = n\widebar{xx} + \Delta$,
where $\Delta$ follows the Laplace distribution according
to the Laplace mechanism, yielding
$$ \Lambda_{DP} = \Lambda_0 + \Lambda (n\widebar{xx} + \Delta). $$
Similarly using $n\widebar{xy}' = n\widebar{xy} + \delta$ with $\delta$
following the Laplace mechanism we obtain
\begin{equation}
  \label{eq:linreg_dp}
  \mu_{DP} = \Lambda_{DP}^{-1} (\Lambda (n\widebar{xy} + \delta) + \Lambda_0 \beta_0).
\end{equation}
As presented in Methods, a more robust alternative is to assign prior
distributions to the precision parameters and then sample the posterior.
This requires using the three sufficient statistics $n\widebar{xx}$, 
$n\widebar{xy}$, and $n\widebar{yy}$ that are perturbed with suitable noise.
The mechanism is presented in detail in Algorithm~\ref{alg:algos} 
and proven to guarantee differential privacy in Theorem~\ref{dp_proof}.
For theoretical analysis, we study the model with fixed precision parameters
and an even privacy budget split between the two needed sufficient statistics.
In Algorithm~\ref{alg:algos} and Theorem~\ref{dp_proof}, this corresponds to 
setting $p_1=p_2=0.5$ and leaving out the unnecessary term $S_{yy}$.

\subsection{The detailed mechanism}

\begin{algorithm}[t]
\caption{Differentially private statistics release}
\label{alg:algos}
 \begin{algorithmic}
\State $p_1+p_2+p_3=1$
\Function{DiffPriSS}{$X$, $Y$, $\epsilon$, $B_x$, $B_y$}
  \State $n=|Y|$, $d=\mathrm{dim}(X)$
  \State $(C,D) = \textproc{project}(X,Y,B_x,B_y)$
  \For{$i \in \{1, \dots, n\}$}
    \For{$j \in \{i, \dots, n\}$}
      \State $P_{ij} = P_{ji} \sim \mathrm{Laplace}\left(0, \frac{d(d+1)B_x^2}{p_1\epsilon}\right)$
    \EndFor
  \EndFor
  \For{$i \in \mbb{I}$} 
    \State $Q_i \sim \mathrm{Laplace}\left(0, \frac{2dB_x B_y}{p_2\epsilon}\right)$
   \EndFor
  \State $R \sim \mathrm{Laplace}\left(0,\frac{B_y^2}{p_3\epsilon}\right)$
  \State $S_{xx} = CC' + P$
  \State $S_{xy} = CD + Q$
  \State $S_{yy} = DD' + R$
\EndFunction
\Function{Project}{$X$, $Y$, $B_x$, $B_y$}
\For{$j=1$ to $n$}
  \For{$i=1$ to $d$} 
    \State $C_{ij} = \max(-B_x, \min(B_x, X_{ij}))$
  \EndFor
    \State $D_{j}= \max(-B_y, \min(B_y, Y_{j}))$
\EndFor
\EndFunction
\end{algorithmic}
\end{algorithm}

The function \textproc{project} in Algorithm~\ref{alg:algos} projects
the data points into a useful space and computes the sufficient
statistics.

\begin{theorem}\label{dp_proof}
  Algorithm \textproc{DiffPriSS} in Algorithm~\ref{alg:algos} is
  $\epsilon$-differentially private.
\end{theorem}
\begin{proof}
(i) $S_{xx} = CC' + P$  is $p_1\epsilon$-differentially private. 

$S_{xx}$ is a symmetric $d \times d$ matrix with $\frac{d (d+1)}{2}$
degrees of freedom.  After \textproc{project} $| C |_\infty \le B_x$ and
the sensitivity of each element
$\Delta (S_{xx})_{ij} = \sup | c_i c_j - c_i' c_j' | \le 2 B_x^2$.
Adding Laplace distributed noise to $(S_{xx})_{ij}$ with
$b = \frac{d(d+1)B_x^2}{p_1\epsilon}$ yields an $\epsilon'$-DP
mechanism with $\epsilon' = \frac{2p_1\epsilon}{d(d+1)}$.
Using basic composition~\cite{Dwork2014} over the $\frac{d (d+1)}{2}$
independent dimensions shows that $S_{xx} = CC' + P$ is
$p_1\epsilon$-differentially private.

(ii) $CD$ is a $d  \times 1$ vector where $d$ is the cardinality of $\mbb{I}$ and each element of $CD$ is computed as follows:
\begin{eqnarray}
\forall i \in \mbb{I},\;\; CD_i=\sum_{j=1}^n C_{ij} D_j,
\end{eqnarray}
where $| C_{ij} | \leq B_x$ and $| D_j | \leq B_y$, and thus the
sensitivity of $CD$ is $2dB_x B_y$. Thus, $S_{xy} = CD + Q$ is
$p_2\epsilon$-differentially private.

(iii) $DD'$ is a scalar computed as
$$DD' = \sum_{j=1}^n D_j^2,$$
where $|D_j| \leq B_y$, and thus the sensitivity of $DD'$ is $B_y^2$. Thus, $S_{yy}=DD'+R$ is $p_3\epsilon$-differentially private.

Therefore, releasing $S_{xx}$, $S_{xy}$, and $S_{yy}$ together by
\textproc{DiffPriSS} is $\epsilon$-differentially private.
\end{proof}

\subsection{Asymptotic consistency and efficiency}

\begin{theorem}\label{thm:efficiency_of_lin_regression}
  Differentially private inference of the posterior mean of the weights
  of linear regression with Laplace mechanism to perturb
  the sufficient statistics is asymptotically consistent with respect
  to the posterior mean.
\end{theorem}

\begin{proof}
  Using Eqs.~(\ref{eq:linreg_np})--(\ref{eq:linreg_dp}) we can evaluate
  \begin{align*}
    \left\| \mu_{DP} - \mu_{NP} \right\|_1 
    &= \left\| \Lambda_{DP}^{-1} (\Lambda (n\widebar{xy} + \delta) + \Lambda_0 \beta_0)
       - \Lambda_{NP}^{-1} (\Lambda n\widebar{xy} + \Lambda_0 \beta_0) \right\|_1 \\
    &\le \left\| \Lambda_{DP}^{-1} (\Lambda (n\widebar{xy} + \delta) + \Lambda_0 \beta_0)
       - \Lambda_{DP}^{-1} (\Lambda n\widebar{xy} + \Lambda_0 \beta_0) \right\|_1\\
    &\phantom{\le } + \left\| \Lambda_{DP}^{-1} (\Lambda n\widebar{xy} + \Lambda_0 \beta_0)
       - \Lambda_{NP}^{-1} (\Lambda n\widebar{xy} + \Lambda_0 \beta_0) \right\|_1 \\
    &= \left\| \Lambda_{DP}^{-1} \Lambda \delta \right\|_1
      + \left\| (\Lambda_{DP}^{-1} - \Lambda_{NP}^{-1})
          (\Lambda n\widebar{xy} + \Lambda_0 \beta_0) \right\|_1 \\
    &= \left\| (\Lambda_0 + \Lambda (n\widebar{xx} + \Delta))^{-1} \Lambda \delta \right\|_1 \\
    &\phantom{=} + \big\| \big[ (\Lambda_0 + \Lambda (n\widebar{xx} + \Delta))^{-1} \\
    &\phantom{= + \big\| \big[ \;}
        - (\Lambda_0 + \Lambda (n\widebar{xx}))^{-1}\big] (\Lambda n\widebar{xy} + \Lambda_0 \beta_0) \big\|_1 \\
    &= \left\| (\Lambda_0 + \Lambda (n\widebar{xx} + \Delta))^{-1} \Lambda \delta \right\|_1\\
    &\phantom{=} + \bigg\| \bigg[ \left(\frac{1}{n}\Lambda_0 + \Lambda \left(\widebar{xx} + \frac{1}{n}\Delta\right)\right)^{-1} \\
    &\phantom{= + \bigg\| \bigg[ \;}
        - \left(\frac{1}{n}\Lambda_0 + \Lambda \widebar{xx}\right)^{-1}\bigg] \left(\Lambda \widebar{xy} + \frac{1}{n}\Lambda_0 \beta_0\right) \bigg\|_1.
  \end{align*}
  Assuming $\widebar{xx} > 0$, the first term clearly approaches 0 as
  $n \rightarrow \infty$.  For the second term, as
  $n \rightarrow \infty$,
  $(\frac{1}{n}\Lambda_0 + \Lambda (\widebar{xx} +
  \frac{1}{n}\Delta))^{-1} \rightarrow (\frac{1}{n}\Lambda_0 + \Lambda
  \widebar{xx})^{-1}$ and as
  $(\Lambda \widebar{xy} + \frac{1}{n}\Lambda_0 \beta_0)$ is bounded, the
  second term also approaches 0 as $n \rightarrow \infty$.  This shows
  that $\mu_{DP}$ converges in probability to $\mu_{NP}$.
\end{proof}

\begin{theorem}
  $\epsilon$-differentially private inference of the posterior mean of
  the weights of linear regression with the Laplace mechanism
  of Algorithm~\ref{alg:algos} to perturb the sufficient statistics
  is asymptotically efficiently private.
\end{theorem}

\begin{proof}
From the proof of Theorem~\ref{thm:efficiency_of_lin_regression} we have
\begin{multline}
  \left\| \mu_{DP} - \mu_{NP} \right\|_1 
    \le \left\| (\Lambda_0 + \Lambda (n\widebar{xx} + \Delta))^{-1} \Lambda \delta \right\|_1\\
    + \left\| \left[ \left(\frac{1}{n}\Lambda_0 + \Lambda \left(\widebar{xx} + \frac{1}{n}\Delta\right)\right)^{-1}
        - \left(\frac{1}{n}\Lambda_0 + \Lambda \widebar{xx}\right)^{-1}\right] \left(\Lambda \widebar{xy} + \frac{1}{n}\Lambda_0 \beta_0 \right) \right\|_1.
  \label{eq:lin_regression_errorbound}
\end{multline}

The first term can be bounded easily as
\begin{align}
  \left\| (\Lambda_0 + \Lambda (n\widebar{xx} + \Delta))^{-1} \Lambda \delta \right\|_1
  &= \left\| (\Lambda^{-1} \Lambda_0 + \Delta + n\widebar{xx})^{-1} \delta \right\|_1 \nonumber\\
  &\le \left\| (\Lambda^{-1} \Lambda_0 + \Delta + n\widebar{xx})^{-1} \right\|_1 \| \delta \|_1 \nonumber\\
  &\le \frac{c_1}{n} \left\| (\widebar{xx})^{-1} \right\|_1 \| \delta \|_1
  \label{eq:linreg_bound1}
\end{align}
where $c_1 > 1$.  The bound is valid for any $c_1 > 1$ as $n$ gets large
enough.

Similarly as in the proof of Theorem~\ref{thm:rate_gaussian_mean},
\begin{equation}
  \label{eq:linreg_delta}
  \| \delta \|_1 \sim \mathrm{Gamma}\left(d, \frac{\epsilon}{4 d B_x B_y}\right).
\end{equation}
Given $\alpha > 0$ we can choose similarly as in the proof of
Theorem~\ref{thm:rate_gaussian_mean}
\begin{equation*}
  C_1 > c_1 F^{-1}(1-\alpha/2; d, \epsilon/(4dB_x B_y)) \left\| (\widebar{xx})^{-1} \right\|_1,
\end{equation*}
where $F^{-1}(x; \alpha, \beta)$ is the inverse distribution function of
the Gamma distribution with shape $\alpha$ and rate $\beta$,
to ensure that 
\begin{equation}
  \label{eq:linreg_prob1}
  \mathrm{Pr}\left\{\| (\Lambda_0 + \Lambda (n\widebar{xx} + \Delta))^{-1} \Lambda \delta \|_1 > \frac{C_1}{n} \right\} 
  < \frac{\alpha}{2}.
\end{equation}

The second term can be bounded as
\begin{multline*}
  \left\| \left[ \left(\frac{1}{n}\Lambda_0 + \Lambda \left(\widebar{xx} + \frac{1}{n}\Delta\right)\right)^{-1}
      - \left(\frac{1}{n}\Lambda_0 + \Lambda \widebar{xx}\right)^{-1}\right] \left(\Lambda \widebar{xy} + \frac{1}{n}\Lambda_0 \beta_0\right) \right\|_1 \\
  = \left\| \left[ \left(\frac{1}{n}\Lambda^{-1} \Lambda_0 + \widebar{xx} + \frac{1}{n}\Delta\right)^{-1}
      - \left(\frac{1}{n}\Lambda^{-1}\Lambda_0 + \widebar{xx}\right)^{-1}\right] \left(\widebar{xy} + \frac{1}{n}\Lambda^{-1}\Lambda_0 \beta_0\right) \right\|_1 \\
  = \frac{1}{n} \left\| \left(\frac{1}{n}\Lambda^{-1} \Lambda_0 + \widebar{xx} + \frac{1}{n}\Delta\right)^{-1} \Delta \left(\frac{1}{n}\Lambda^{-1}\Lambda_0 + \widebar{xx}\right)^{-1} \left(\widebar{xy} + \frac{1}{n}\Lambda^{-1}\Lambda_0 \beta_0\right) \right\|_1\\
  \le \frac{1}{n} \left\| \left(\frac{1}{n}\Lambda^{-1} \Lambda_0 + \widebar{xx} + \frac{1}{n}\Delta\right)^{-1} \Delta \left(\frac{1}{n}\Lambda^{-1}\Lambda_0 + \widebar{xx}\right)^{-1} \right\|_1 \left\| \widebar{xy} + \frac{1}{n}\Lambda^{-1}\Lambda_0 \beta_0 \right\|_1 \\
  \le \frac{1}{n} \left\| \left(\frac{1}{n}\Lambda^{-1} \Lambda_0 + \widebar{xx} + \frac{1}{n}\Delta\right)^{-1} \right\|_1
  \left\| \Delta \right\|_1 \\
  \left\| \left(\frac{1}{n}\Lambda^{-1}\Lambda_0 + \widebar{xx}\right)^{-1} \right\|_1
  \left\| \widebar{xy} + \frac{1}{n}\Lambda^{-1}\Lambda_0 \beta_0 \right\|_1 \\
  \le \frac{c_2}{n} \left\| \left(\widebar{xx}\right)^{-1} \right\|_1
  \left\| \Delta \right\|_1
  \left\| \left(\widebar{xx}\right)^{-1} \right\|_1
  \left\| \widebar{xy} \right\|_1 =: \frac{c_2}{n} \mathcal{B}_2,
\end{multline*}
where similarly as in Eq.~\eqref{eq:linreg_bound1}, the bound is valid
for any $c_2>1$ as $n$ gets large enough.
Here $\|\Delta\|_1$ is the $l_1$-norm of the matrix $\Delta$ that
whose elements follow the Laplace distribution
$\Delta_{ij} \sim \mathrm{Laplace}(0, \frac{2d(d+1)B_x^2}{\epsilon})$.
We can bound it as
\begin{equation*}
  \| \Delta \|_1 = \max_i \| \Delta_{:i} \|_1,
\end{equation*}
where $\Delta_{:i}$ are the row vectors of $\Delta$ and the latter is
the vector $\ell_1$-norm.  Similarly as in Eq.~\eqref{eq:linreg_delta}
we have
\begin{equation}
  \label{eq:linreg_delta2}
  \| \delta \|_1 \sim \mathrm{Gamma}\left(d, \frac{\epsilon}{2 d (d+1) B_x^2}\right)
\end{equation}
and as above given $\alpha > 0$ we can choose
\begin{equation*}
  C_2 > c_2 F^{-1}(1-\alpha/2; d, \epsilon/(2d(d+1)B_x^2)) \left\| \left(\widebar{xx}\right)^{-1} \right\|_1^2 \left\| \widebar{xy} \right\|_1,
\end{equation*}
where $F^{-1}(x; \alpha, \beta)$ is the inverse distribution function of
the Gamma distribution to ensure that
\begin{equation}
  \label{eq:linreg_prob2}
  \mathrm{Pr}\left\{ \mathcal{B}_2 > \frac{C_2}{n} \right\} 
  < \frac{\alpha}{2}.
\end{equation}

Combining Eqs.~(\ref{eq:linreg_prob1}) and (\ref{eq:linreg_prob2})
shows that
\begin{equation}
  \label{eq:linreg_efficiency}
  \mathrm{Pr}\left\{ \left\| \mu_{DP} - \mu_{NP} \right\|_1 > \frac{C_1 + C_2}{n} \right\} < \alpha.
\end{equation}
\end{proof}

\subsection{Convergence rate}

Using Chebysev's inequality together with Eq.~\eqref{eq:linreg_delta}
we can show that with high probability
\begin{equation*}
  \| \delta \|_1 = \mathcal{O}\left(\frac{d^2 B_x B_y}{\epsilon}\right)
\end{equation*}
and thus
\begin{equation}
  \left\| (\Lambda_0 + \Lambda (n\widebar{xx} + \Delta))^{-1} \Lambda \delta \right\|_1
  = \mathcal{O}\left(\frac{d^2 B_x B_y \left\| (\widebar{xx})^{-1} \right\|_1}{n \epsilon}\right).
  \label{eq:lin_regression_bound1}
\end{equation}

Similarly for the second term we obtain
\begin{equation}
  \mathcal{B}_2 = \mathcal{O}\left( \frac{d^3B_x^2 \left\| \left(\widebar{xx}\right)^{-1} \right\|_1^2 \left\| \widebar{xy} \right\|_1}{\epsilon}  \right).
  \label{eq:lin_regression_bound2}
\end{equation}
Combining
Eqs.~(\ref{eq:lin_regression_errorbound})--(\ref{eq:lin_regression_bound2})
yields
\begin{equation*}
  \left\| \mu_{DP} - \mu_{NP} \right\|_1 
  = \mathcal{O}\left( \frac{d^2 B_x B_y \|\widebar{xx}^{-1}\|_1 
      + d^3 B_x^2 \left\| \left(\widebar{xx}\right)^{-1} \right\|_1^2 \left\| \widebar{xy} \right\|_1}{n \epsilon} \right)
\end{equation*}
with high probability.

\begin{figure}[htb]
  \centering
  \centerline{\includegraphics{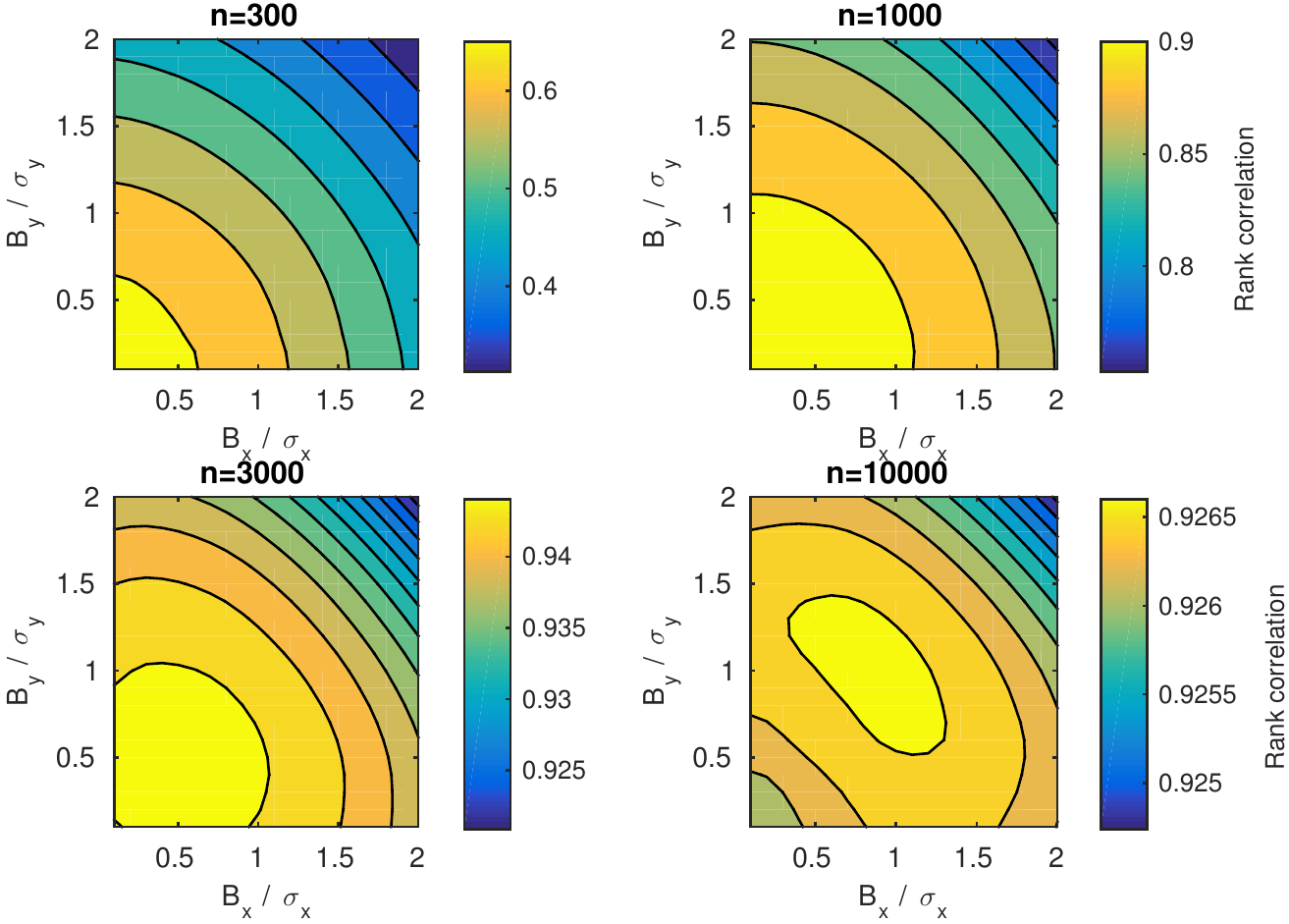}}
  \caption{Illustration of the effect of projecting the outliers in
    linear regression, for different sample sizes $n$ with 10-dimensional
    synthetic data, evaluated by Spearman's rank correlation between the
    predicted and true values. The $x$ and $y$ axes denote the projection
    thresholds as
    a function of standard deviations of data. Top right corner
    illustrates projection threshold at 2 standard deviations,
    no outlier projection would
    be further to top right. Higher
    values (yellow) are better. The result
    illustrates a clear benefit from the projection for moderate
    sample sizes, but the benefit decreases for really large sample
    sizes.}
  \label{fig:clipping_effect}
\end{figure}

\begin{figure}
\centering
\centerline{\includegraphics{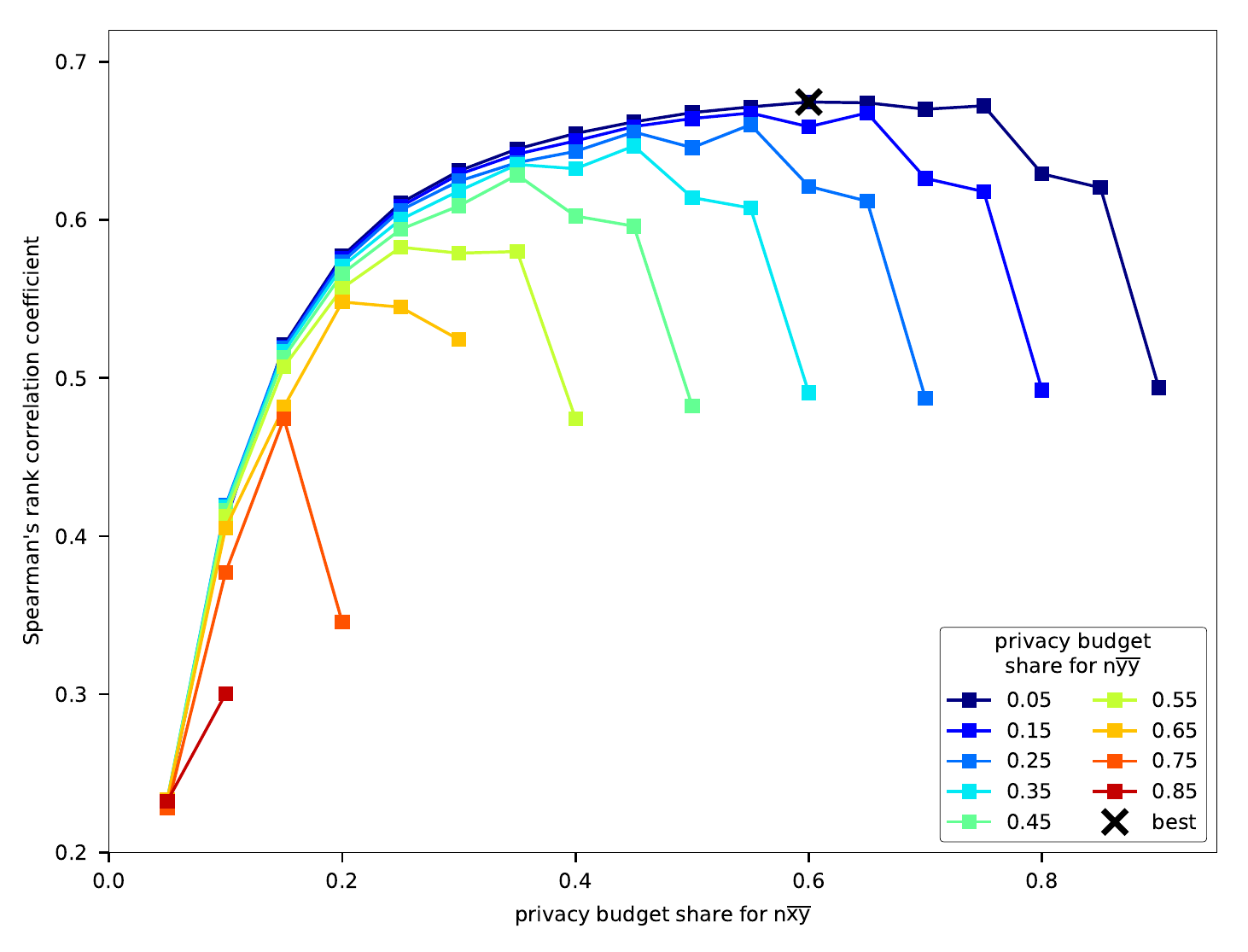}}
\caption{\textbf{Optimal privacy budget split between sufficient statistics.} 
  Accuracy on a synthetic data set improves as a bigger proportion 
  of the fixed privacy budget is assigned for $n\protect\widebar{xy}$. 
  The best performance is achieved by assigning term $n\protect\widebar{yy}$ 
  the smallest proportion 5\%, term $n\protect\widebar{xy}$ a large 60\% 
  proportion, and term $n\protect\widebar{xx}$ the remaining 35\% proportion 
  of the privacy budget.}
\label{fig:budgetsplit}
\end{figure}

\begin{figure}
\centering
\centerline{\includegraphics{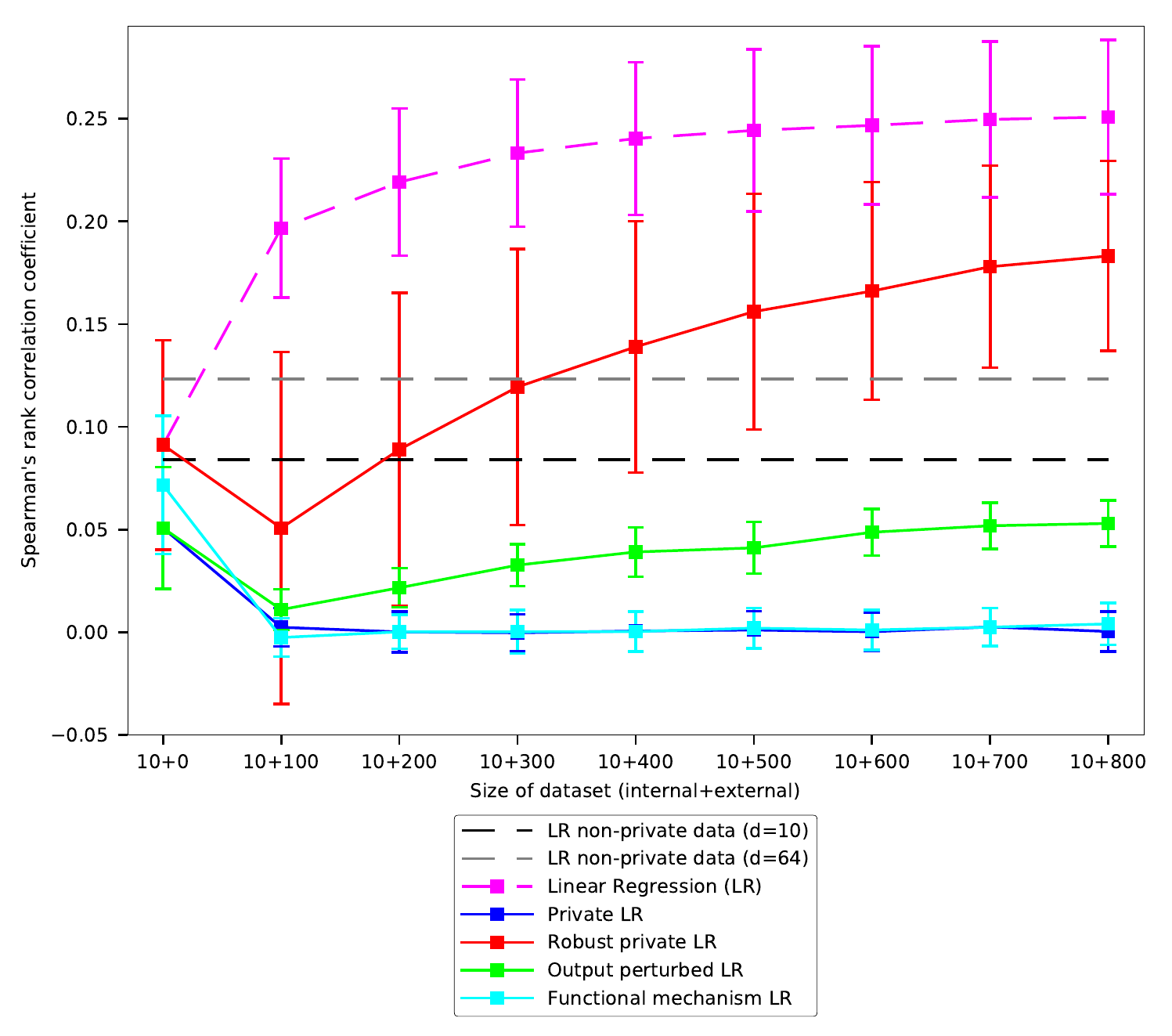}}
\caption{This is a complement to Figure \ref{fig:curves} with more
  stringent privacy. Here we show Spearman's rank correlation
  coefficients ($\rho$)
  between the measured ranking of the cell lines and the ranking
  predicted by the models using $\epsilon=1$.
  The baselines (horizontal dashed lines) are learned on 10
  non-private data points; the
  private algorithms additionally have privacy-protected data
  (x-axis). The non-private algorithm (LR) has the same amount of
  additional non-privacy-protected data.
  All methods use 10-dimensional data except purple baseline
  showing the best performance with 10 non-private data points.
  The results are averaged
  over all drugs and 50-fold Monte Carlo cross-validation; error bars
  denote standard deviation over 50 Monte Carlo repeats.
  The result shows that more data are needed for good prediction
  performance under more stringent privacy.}
\label{fig:curves_eps1}
\end{figure}

\end{document}